\documentclass[lettersize,journal]{IEEEtran}
\usepackage{amsmath,amsfonts}

\usepackage{algorithmic}
\usepackage{algorithm}

\usepackage{array}

\usepackage[caption=false,font=scriptsize,labelfont=sf,textfont=sf]{subfig}

\usepackage{textcomp}
\usepackage{stfloats}
\usepackage{url}
\usepackage{verbatim}
\usepackage{graphicx}
\usepackage{cite}
\usepackage{diagbox}

\usepackage{epstopdf}

\usepackage{xurl}
\usepackage{etoolbox}
\apptocmd{\UrlBreaks}{\do\f\do\m}{}{}

\usepackage{amsthm}
\newtheorem{definition}{Definition}
\newtheorem{proposition}{Proposition}

\newtheorem{lemma}{Lemma}

\usepackage{hhline}

\usepackage{enumitem}

\begin{document}

\title{Likelihood-based Sensor Calibration using Affine Transformation}

\author{Rüdiger~Machhamer, Lejla~Begic~Fazlic, Eray~Guven, David~Junk,  Gunes~Karabulut~Kurt, \IEEEmembership{Senior Member, IEEE}, Stefan~Naumann, Stephan~Didas, Klaus-Uwe~Gollmer, Ralph~Bergmann, Ingo~J.~Timm, and \IEEEauthorblockN{Guido~Dartmann, \IEEEmembership{Senior Member, IEEE}\\
This work has been published by IEEE under DOI: 10.1109/JSEN.2023.3341503.}

\thanks{This project was partly funded by the German Federal Ministry of Food and Agriculture (BMEL) projects "KI-Pilot" under Grant 2820KI001, and "PINOT" under Grant 28DK107G20, and the German Federal Ministry for Economic Affairs and Climate Action (BMWK) project "EASY" under Grant 01MD22002D, and the German Federal Ministry for the Environment, Nature Conservation, Nuclear Safety, and Consumer Protection (BMUV) project "KIRA" under Grant 67KI32013B, and the Ministry for Science and Health of Rhineland-Palatinate (MWG) as part of the research training group (Forschungskolleg) "AI-CPPS".}
\thanks{Corresponding author is R.~Machhamer (e-mail: iss@umwelt-campus.de). R.~Machhamer, L.~B.~Fazlic, D.~Junk, S.~Naumann, S.~Didas, K.-U.~Gollmer, and G.~Dartmann are with the Institute for Software Systems (ISS), Trier University of Applied Sciences, Environmental Campus Birkenfeld, Germany.} 
\thanks{E.~Guven and G.~Karabulut Kurt are with the Poly-Grames Research Center, Polytechnique Montr\'eal, Canada.}
\thanks{R.~Bergmann and I.~J.~Timm are with Trier University, Germany. \protect\\ \\ }
}

\maketitle

\begin{abstract}
An important task in the field of sensor technology is the efficient implementation of adaptation procedures of measurements from one sensor to another sensor of identical design. One idea is to use the estimation of an affine transformation between different systems, which can be improved by the knowledge of experts. This paper presents an improved solution from Glacier Research that was published back in 1973. The results demonstrate the adaptability of this solution for various applications, including software calibration of sensors, implementation of expert-based adaptation, and paving the way for future advancements such as distributed learning methods. One idea here is to use the knowledge of experts for estimating an affine transformation between different systems. We evaluate our research with simulations and also with real measured data of a multi-sensor board with 8 identical sensors. Both data set and evaluation script are provided for download. The results show an improvement for both the simulation and the experiments with real data.
\end{abstract}


\section{Introduction}\label{cha:Introduction}
\IEEEPARstart{I}{n} Internet of things applications, as machine learning for sensor systems, there are often systems with similar behavior, e.g. two sensors measuring the same data or the same machine in different process settings. In this context, the transformation between Euclidean data spaces of two systems can be described by an affine transformation. Figure~\ref{fig:Intro2} illustrates the concept of expert-assisted learning. We assume that experts know estimates of several example data points in both systems and can associate some data points of a similar machine with these measurements/estimates. Using these noisy measurements, we want to learn the transformation from system~\(1\) to system~\(2\). Many examples of this exist in practice, e.g. drifts and transformations of identically constructed sensors. These sensors are subject to non-identical fabrication defects. That is, in feature space, the data of two sensors are shifted or slightly rotated. 

A common practice to align the data for e.g. machine learning purposes is feature-wise normalization. Figure~\ref{fig:Intro1a} shows raw measurement values from an example application of an artificial nose analyzing apple juice during a heating process using 8 sensors. The black and blue dots show the data of the relative resistance \([\frac{\Omega}{\Omega}]\) values from sensor~1 and sensor~2 at a sensor plate temperature of \(200\)~°C and \(400\)~°C. The red dots represent the estimation for sensor~\(2\), applying the affine transformation (AT) on the data of sensor~\(1\). The data collection process is described in detail in section \ref{RealDataSimulation}.
Figure~\ref{fig:Intro1b} shows the feature-wise normalized values. It can be seen that normalization already improves the similarity between the data spaces and that the data spaces can be further aligned if the transformation is applied between the two data spaces. Data measured with one sensor can be mapped to data in the data space of the other sensor. Similar practices can be used, to adjust built-in bias in sensors or to perform software-based maintenance for drifts on aged or poisoned sensors.
\begin{figure}[!t]
    \centering
    \includegraphics[scale=0.5]{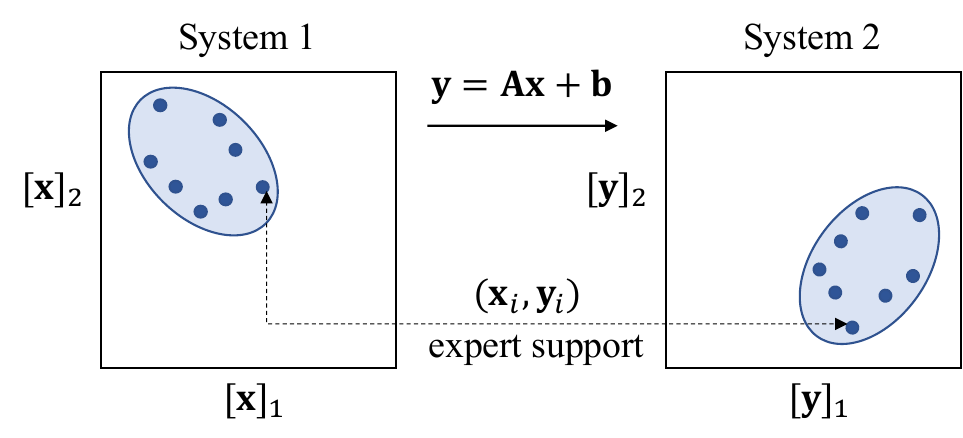}
         \caption{Concept of expert-based learning. Some model instances from sensor~1 are matched by experts with their representatives in sensor~2, the remaining instances from sensor~1 are calculated using the estimated AT, or vice versa. \vspace{-1em}}
         \label{fig:Intro2}
\end{figure}
By combining these two methods, we lay the foundation for improving transformation estimation through expert knowledge. This can be used in distributed learning systems by checking the computed transformation between similar systems by an expert, and iteratively triggering a recomputation or further adaptation of the transformation in case of incorrect predictions. The given example of 8 sensors measuring the same fluid, will provide similar data as if 8 sensors were measuring 8 similar fluids at different locations, which would correspond to a distributed learning use case.

In this paper, we present an approach to estimate the transformation between two data spaces by an AT. Referring to the example with different sensors, an expert can perform the measurement with all sensors under laboratory conditions. Then these measurements are used to estimate the transformation between the two sensors (systems). The basis of this approach is the problem of estimating an AT.
\begin{figure}
\centering
\subfloat[Raw measurement data of an experiment with an electronic nose.\label{fig:Intro1a}]{\includegraphics[width=0.45\textwidth]{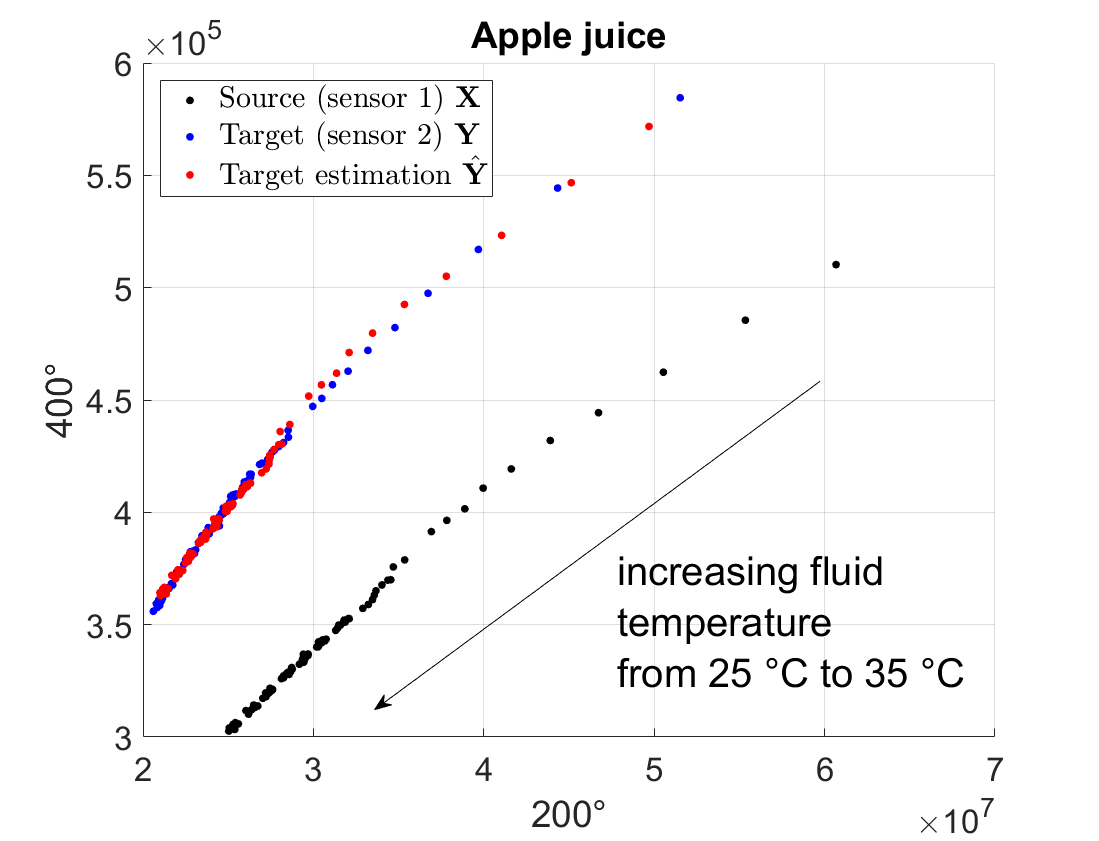}}%

\subfloat[Normalized measurement data of an experiment with an electronic nose.\label{fig:Intro1b}]{\includegraphics[width=0.45\textwidth]{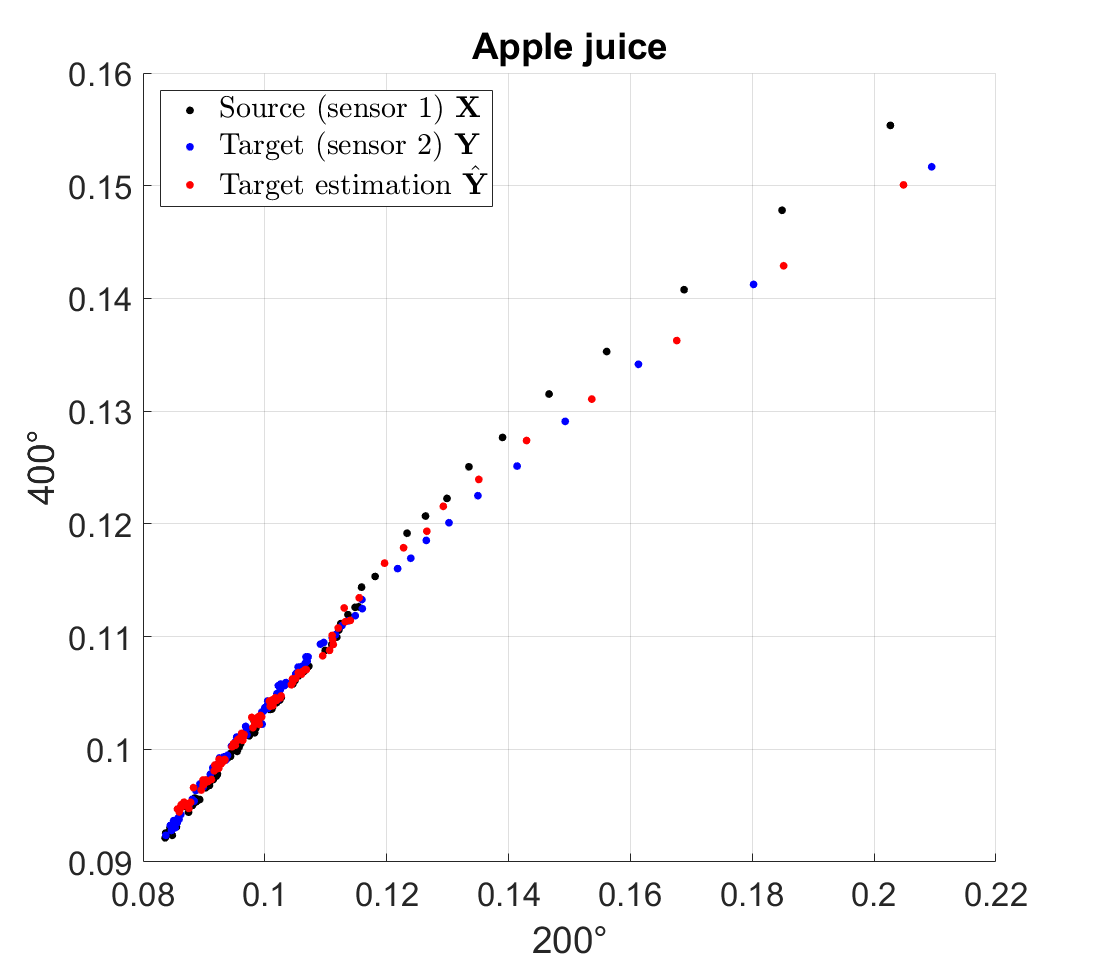}}%

\caption{Comparison of (a) raw and (b) feature-wise normalized data.}\label{fig:Intro1}
\end{figure}

Inspired by a statistical problem in glaciology, Gleser and Watson~\cite{IEEEhowto:Gleser} obtained maximum likelihood estimators (MLE) of the unknown features and an AT. As a continuation of this research, the authors in \cite{IEEEhowto:Gleser} showed that a MLE for the unknown parameter exists, and gained MLE for the model parameters in closed form.

This paper is organized as follows: in section~\ref{SystemModel}, we derive the system model and formulate the estimation problem of the AT. In section~\ref{Algorithms}, we present the algorithms, which we in part have selected from the literature \cite{IEEEhowto:Gleser}. Section~\ref{Gleser} summarizes the solution of Gleser and Watson. We show in section~\ref{Regression} that the stationary points of the parameters of the transformation lead to similar results as the least squares solution, which can be combined with Gleser and Watson~to provide better estimation of the AT. Section~\ref{Results} summarizes the simulation results based on Monte Carlo simulations and shows the results in a real multi-sensor setup.

\noindent Our main contribution is outlined as follows:
\begin{enumerate}[leftmargin=*]
    \item We show that the results of an old but fundamental work~\cite{IEEEhowto:Gleser} in glaciology can be used for the estimation of an AT in expert-supported learning. In \cite{IEEEhowto:Gleser}, the authors derive a solution for the estimation of a linear transformation. Their solution allows to estimate the data in system 1, variance and the transformation. By using the well-known augmented matrix~\cite{Howto:Affine}, this result can be easily improved by denoising using the eigenvectors and extended to an AT.
    \item Furthermore, the authors already showed that a simple solution directly follows from the multivariate regression analysis. We are interested in the estimation of the AT and we additionally derive the gradient of the MLE to get a stationary point of the parameters of the transformation, which leads to a similar simplificated \emph{least square solution} as in \cite{IEEEhowto:Gleser}.
    \item We prove that both solutions are connected by an eigenvalue decomposition. By simulation, we show that a combination of both solutions provides better estimates for computing the parameters of the AT compared to the original approach from Gleser and Watson~\cite{IEEEhowto:Gleser}.
    \item We also show an experiment with a multi-sensor Bosch BME688 development kit board (Fig.~\ref{fig:bme688}) to evaluate the estimation error of the different methods and provide the example dataset and evaluation code.
    \item Finally, we briefly present an idea for a novel concept for distributed learning using transformation with the support of expert knowledge.
\end{enumerate}

\subsection{Related Work}
\noindent More general approach for the MLE procedure  for linear model with errors in variables is proposed by authors in \cite{IEEEhowto:Rock,IEEEhowto:Gleser1}. In~\cite{IEEEhowto:Chan}, the relationship between maximum likelihood and generalize least squares approaches are examined and some of the results when estimating in \emph{error in variables} models are extended. The linear transformations trained in maximum likelihood sense on adaptation data for hidden Markov model based speech recognition is researched in \cite{IEEEhowto:Gales,IEEEhowto:Psutka}. The authors in \cite{IEEEhowto:Balakrishnama} discuss and compare the optimal linear estimation and modified MLE, where they show the better results for modified MLE from the point of bias and means square error in comparison with optimal linear model. The maximum likelihood method is used in regression analysis of linear transformation model and efficient and consistent estimator is developed for  different data case of cohorts \cite{IEEEhowto:Yao,IEEEhowto:Zeng}. MLE optical flow model in monocular optical flow field is developed and presented by authors in \cite{IEEEhowto:Liu}. 

Domain adaptation using maximum likelihood linear transformation for speaker verification is presented by authors in \cite{IEEEhowto:Misra,IEEEhowto:Wang}. The authors in \cite{IEEEhowto:Xingwei} investigate the transformed linear regression with non-normal error distributions where they develop a MLE approach and provide the almost optimal conditions on the error distribution. The authors in \cite{IEEEhowto:Chen}, by using Student's t-distribution noise model, introduce a distributed maximum-likelihood based state estimation approach in domain of power systems. The authors in~\cite{IEEEhowto:Reisizadeh} propose a Distributed Learning framework robust to affine distribution shifts (FLRA) for non-i.i.d. personalization to implement efficient fistributed learning to generate an averaged global model using ATs. In~\cite{IEEEhowto:Faridee}, AT parameters are used to analyze the domain heterogeneity for human activity recognition. The research by \cite{Bib:Vito} introduces a cost-effective calibration method for Low Cost Air Quality Monitoring Systems, which ensures accurate measurements for widespread deployment. This approach reduces calibration costs and enhances the accessibility of precise readings, making it suitable for various IoT Air Quality monitoring devices. The authors in \cite{Bib:Zhang} propose TDACNN, a target-domain-free domain adaptation convolutional neural network designed to mitigate the effects of ill-posed gas sensor drift in E-nose applications. The introduction of an additive angular margin softmax loss during training and the use of the MMD-based ensemble method contribute to improved feature generalization and interclass distance enlargement without the need for participation of target domain data. The authors in \cite{Bib:Fonollosa} built five twin sensing units with eight MOX gas sensors each, exploring sensor tolerance and the impact of sensor mismatch, and they demonstrated the effectiveness of the Direct Standardization (DS) transformation for calibration transfer, addressing drift and variability issues. Tao~\emph{et al.}~\cite{Bib:Yang}~propose a domain correction based on the kernel transformation (DCKT) for drift compensation. 
They show that transformation can lead to better classification results in drift-prone sensor setups by modifying source and target domain. Authors in \cite{Bib:Song,Bib:Chong} addresses flaws in Maximum Likehood Calibration Array (MLCA), proposing a modified cost function and emphasizing improved convergence rates, while our proposed approach focuses on expert-assisted learning between systems, combining direct estimation with a solution from Gleser and Watson.
\begin{figure}
    \centering
    \includegraphics[scale=0.14]{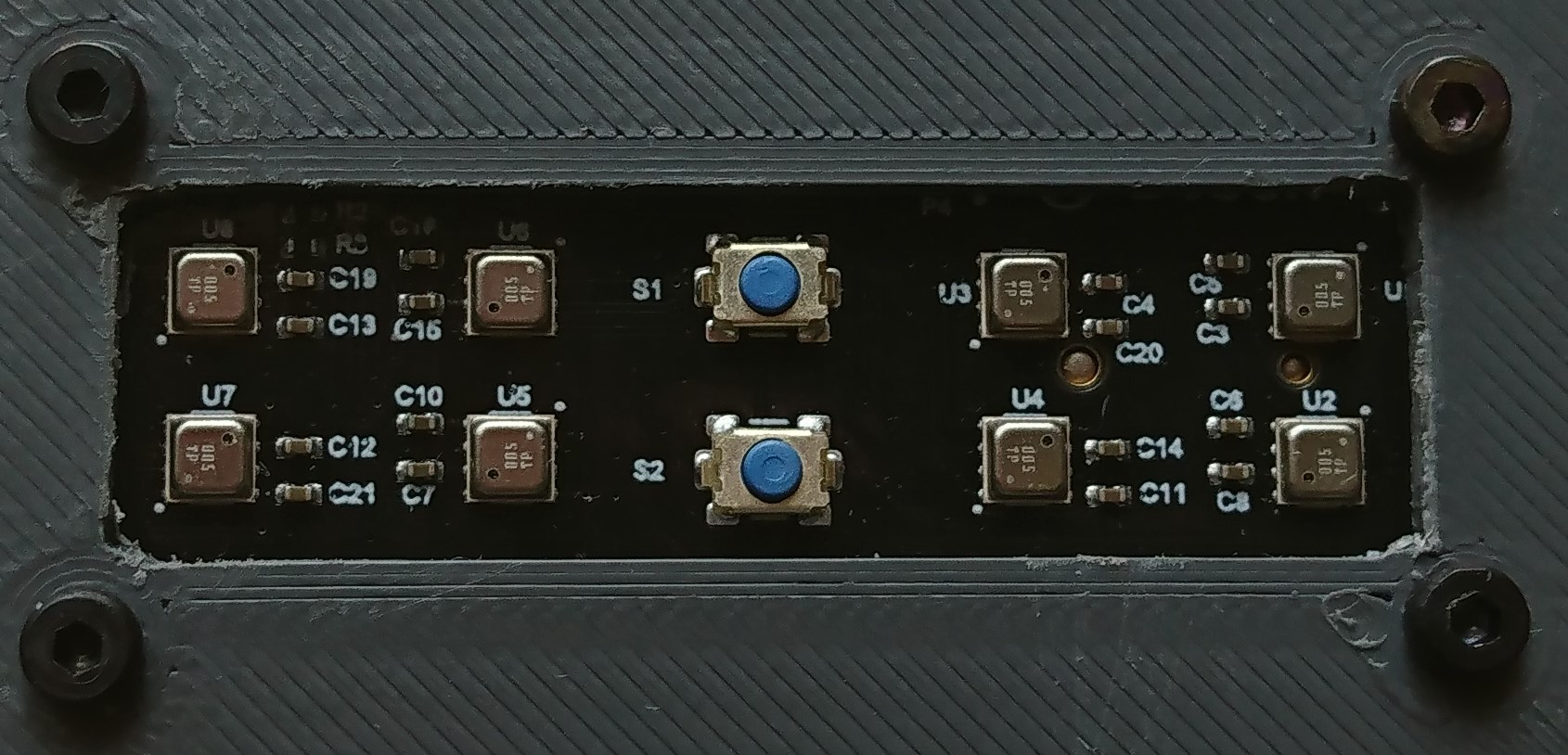}
    \caption{Multisensor board of the Bosch BME688 development kit, which provides 8 identical constructed sensors \cite{Bib:Bosch}.}
    \label{fig:bme688}
\end{figure}
\section{System Model}\label{SystemModel}
\noindent In this section, we define the mathematical basis in order to develop the
corresponding algorithms in Section~\ref{Algorithms}. Fig.~\ref{System1} shows an example of the transformation for the 2-dimensional case. This figure presents a simulation of the data vectors \(\mathbf{v}\in\mathbb{R}^2\) in an Euclidean space. The black data points are the measured points \(\mathbf{X}\) from system~\(1\) and the red data points \(\boldsymbol{\Theta}\) are the estimation of these data points. The green data points \(\mathbf{Y}\) represent the (noisy) transformation of the black data points \(\mathbf{X}\) into system~\(2\) and the blue data points \(\hat{\mathbf{Y}}\) are the origin estimation of this transformed data.
As mentioned in \cite{IEEEhowto:Gleser}, we consider two different systems in which the same features are measured. A single measurement data vector of system~\(1\) is denoted by \(\mathbf{x}_i\) and the single data measurement vector of system~\(2\) by \(\mathbf{y}_i\) respectively, where \(i\in \{1,...,n\}\). Let \(\mathbf{x}_i \in \mathbb{R}^q\) and \(\mathbf{y}_i \in \mathbb{R}^q\), then the data vectors from system~\(1\) are mapped to the data vectors in system~\(2\) by the transformation \(\mathbf{y}_i=\mathrm{T}(\mathbf{x}_i)\). 

We can summarize in the following definition:
\begin{figure}
     \centering
      \includegraphics[scale=0.3]{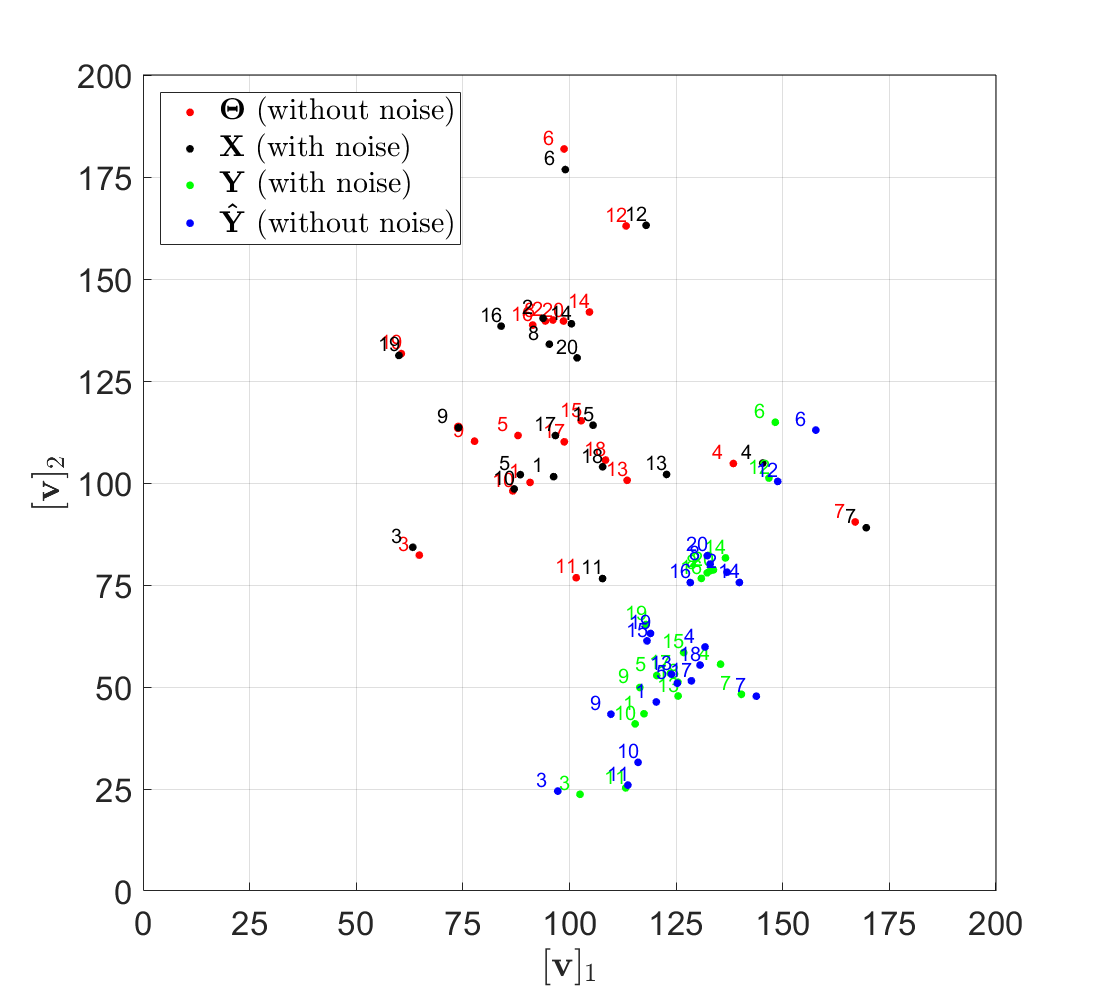}
         \caption{Simulated data of~\(\mathbf{X}\),~\(\boldsymbol{\Theta}\), \(\mathbf{Y}\) and \(\hat{\mathbf{Y}}\). The \(x\)-axis denotes the first component of a data vector \(\mathbf{v}\) and the \(y\)-axis the second component.}
         \label{System1}
 \end{figure}
\begin{definition}\label{Transformation}
Affine Transformation -- Let \(\mathbf{x}\in \mathbb{R}^q\), \(\mathbf{A}\in \mathbb{R}^{q \times q}\) and \(\mathbf{b}\in\mathbb{R}^q\), the AT is given by \(\mathrm{T}(\mathbf{x})= \mathbf{A} \mathbf{x} + \mathbf{b}\). 
\end{definition}
We assume that the measurement in system~\(1\) is noisy. The measurement function in system~\(1\) is given by~(\ref{measurement1}):
\begin{equation}\label{measurement1}
    \mathbf{x}_i=\boldsymbol{\theta}_i+\mathbf{m}_i.
\end{equation}
We also assume that the measurement in system~\(2\) is noisy. The measurement function in system~\(2\) is therefore given by~(\ref{measurement2_1}):
\begin{equation}\label{measurement2_1}
    \mathbf{y}_i=\mathrm{T}(\boldsymbol{\theta}_i)+\mathbf{n}_i.
\end{equation}

The vectors \(\mathbf{m}_i\), \(\mathbf{n}_i\) are zero mean white Gaussian noise with \(\mathbf{m}_i\in \mathcal{N}(\mathbf{0},\sigma \mathbf{I})\), \(\mathbf{n}_i\in \mathcal{N}(\mathbf{0},\sigma \mathbf{I})\).
With Def.~\ref{Transformation}, we can rewrite the measurement in system~\(2\) as~(\ref{measurement2}):
\begin{equation}\label{measurement2}
    \mathbf{y}_i=\mathbf{A} \boldsymbol{\theta}_i + \mathbf{b}+\mathbf{n}_i
\end{equation}
We assume \(\mathbf{m}_i\) and \(\mathbf{n}_i\) are statistically independent.
The covariance matrix of \(\mathbf{n}_i\) and \(\mathbf{m}_i\) is \(\mathrm{Cov}[\mathbf{m}_i]=\mathrm{Cov}[\mathbf{n}_i]=\sigma^2 \mathbf{I}\).
We can summarize all measurements to the measurement matrices \(\mathbf{X}_{M}=[\mathbf{x}_1,...,\mathbf{x}_n]\), and \(\mathbf{Y}_{M}=[\mathbf{y}_1,...,\mathbf{y}_n]\), and the estimated origin matrix \(\boldsymbol{\Theta}_{E}=[\boldsymbol{\theta}_1,...,\boldsymbol{\theta}_n]\).
\section{Algorithms}\label{Algorithms}
\noindent The solution of Gleser and Watson~\cite{IEEEhowto:Gleser} was only derived for a linear transformation without a translation vector \(\mathbf{b}\). 
To introduce the estimation with the translation vector \(\mathbf{b}\), we extend our input using the well known definition of the augmented matrix in~(\ref{Augmented})
\begin{equation}\label{Augmented}
\mathbf{B}=\left[\begin{array}{c|c}
\mathbf{A} & \mathbf{b} \\
\mathbf{0}^T & 1 \\
\end{array} \right]
\end{equation}
and the new augmented vectors \(\hat{\mathbf{y}}_i=[\mathbf{y}_i^T,\beta_i]^T\), \(\hat{\mathbf{x}}_i=[\mathbf{x}_i^T,\alpha_i]^T\), \(\hat{\boldsymbol{\theta}}_i=[\boldsymbol{\theta}_i^T,\gamma_i]^T\), $\gamma_i=1$, $\forall i=1...n$, \(\hat{\mathbf{m}_i}=[\mathbf{m}_i^T,\mu_i]^T\), and \(\hat{\mathbf{n}_i}=[\mathbf{n}_i^T,\nu_i]^T\) the estimation problem of~(\ref{measurement2}) and~(\ref{measurement1}) can be represented by equations~(\ref{measurement3}) and (\ref{measurement4}):
\begin{align}\label{measurement3}
\hat{\mathbf{x}}_i&= \hat{\boldsymbol{\theta}}_i + \hat{\mathbf{m}}_i,\\
\label{measurement4}
    \hat{\mathbf{y}}_i&=\mathbf{B} \hat{\boldsymbol{\theta}}_i + \hat{\mathbf{n}}_i.
\end{align}
In case of $\alpha_i=1$, $\beta_i=1$, $\gamma_i=1$, $\mu_i=0$, and $\nu_i=0$, we have tranformed the affine transformation into a linear form. 
\begin{figure}
     \centering
      \includegraphics[scale=0.3]{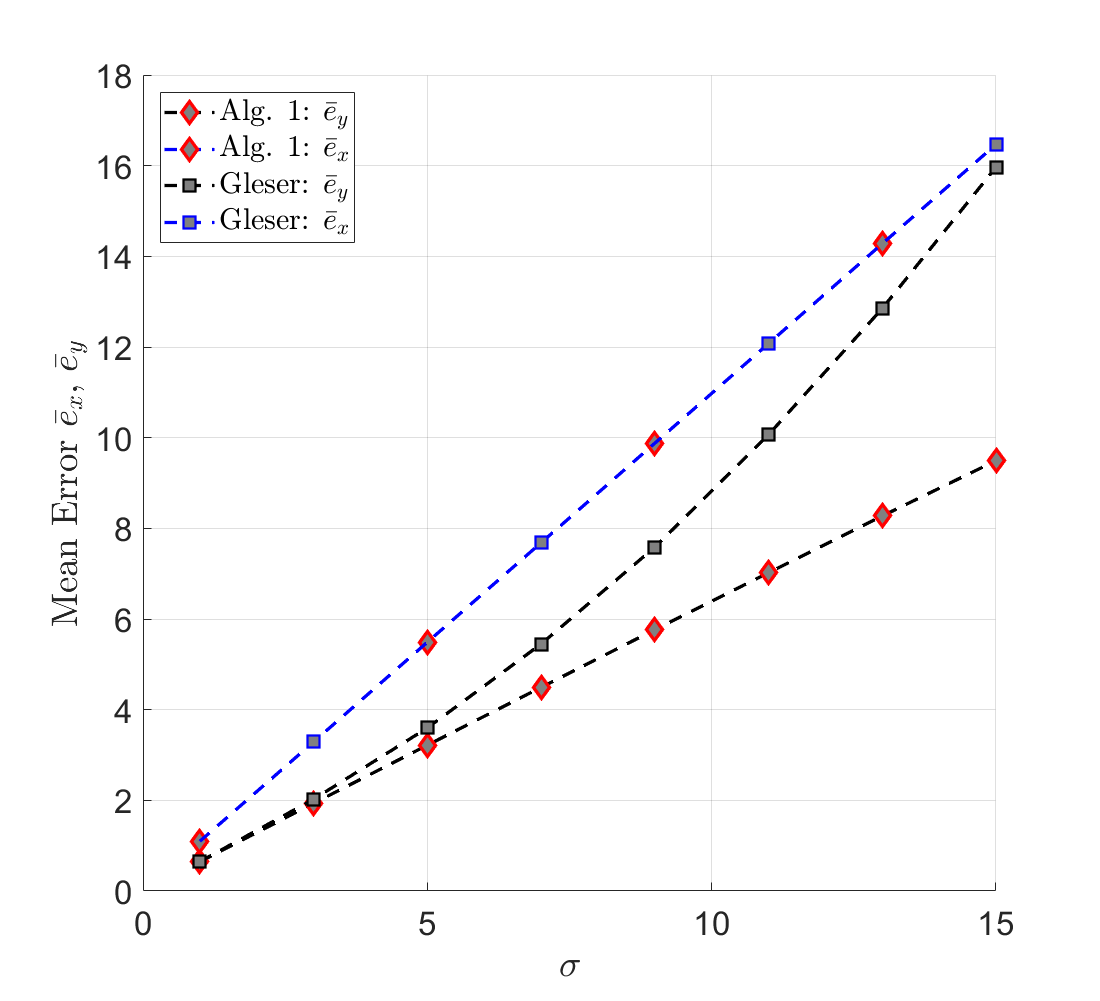}
         \caption{Simulation result, mean errors \(\bar{e}_{x}\) and \(\bar{e}_{y}\) for the augmented implementation of Gleser and Watson~compared to the augmented Alg.~\ref{Alg1} as a function of the standard deviation \(\sigma\).}
         \label{fig:gleserAU}
\end{figure}
Following the same steps as in \cite{IEEEhowto:Gleser}, we can define the same joint likelihood function:
\begin{definition}\label{Likelihood}
The joint likelihood function \cite{IEEEhowto:Gleser} -- with \(p=q+1\), \(n\geq 2p\) and \(\mathbf{X}=[\hat{\mathbf{x}}_1,...,\hat{\mathbf{x}}_n]\), \(\boldsymbol{\Theta}=[\hat{\boldsymbol{\theta}}_1,...,\hat{\boldsymbol{\theta}}_n]\), and \(\mathbf{Y}=[\hat{\mathbf{y}}_1,...,\hat{\mathbf{y}}_n]\) of dimension \(n\times p\) the joint likelihood of \(\mathbf{X}\), \(\mathbf{Y}\) is given by~(\ref{jointlikelihood}):
\begin{equation}\label{jointlikelihood}
    p(\mathbf{X},\mathbf{Y} \vert \boldsymbol{\Theta}, \mathbf{B}, \sigma^2)=(2\pi \sigma^2)^{-n p} e ^{-\frac{1}{2\sigma^2} f(\mathbf{X},\mathbf{Y},\boldsymbol{\Theta})}
\end{equation}
leading to~(\ref{functionf}),
\begin{equation}\label{functionf}
    f(\mathbf{X},\mathbf{Y},\boldsymbol{\Theta})=\mathrm{tr}[(\mathbf{X}-\boldsymbol{\Theta})(\mathbf{X}-\boldsymbol{\Theta})^T]+\mathrm{tr}[(\mathbf{Y}-\mathbf{B}\boldsymbol{\Theta})(\mathbf{Y}-\mathbf{B}\boldsymbol{\Theta})^T].
\end{equation}
\end{definition}
\subsection{Solution proposed by Gleser and Watson}\label{Gleser}
\noindent With Definition~\ref{Likelihood} MLE of \(\boldsymbol{\Theta}\) and \(\mathbf{B}\) of \cite{IEEEhowto:Gleser} can be computed by the observation of the following Lemma:
\begin{lemma}\label{LemmaGleser}
Let \(\boldsymbol{\lambda}\) be the diagonal matrix of the \(p\) largest eigenvalues and \(\mathbf{U}\) be the matrix of all corresponding eigenvectors of \(\mathbf{X}^T\mathbf{X}+\mathbf{Y}^T\mathbf{Y}\),
then the estimates of \(\boldsymbol{\Theta}\) and \(\mathbf{B}\) are given by~(\ref{eigenTheta}) and~(\ref{eigenB}): 
\begin{align}\label{eigenTheta}
    \boldsymbol{\Theta}^{\ddagger}&=[\mathbf{U}\mathbf{U}^T\mathbf{X}^T]^T\\\label{eigenB}
    \mathbf{B}^{\ddagger}&=\mathbf{Y}\boldsymbol{\Theta}^T(\boldsymbol{\Theta}\boldsymbol{\Theta}^T)^{-1}
\end{align}
\end{lemma}
The proof is given in \cite{IEEEhowto:Gleser} in Eqs. (2.4), (2.6) and (2.9).
\subsection{Multivariate Regression Approach}\label{Regression}
\noindent 
Important for an implementation on low-power IoT-devices is a balancing of the computation complexity of the algorithms. Therefore, we also investigate an alternative solution in this section. The solution of Gleser and Watson in Lemma \ref{LemmaGleser} needs an eigenvector decomposition. We are interested in an estimation of the AT from \(\mathbf{X}\) to \(\mathbf{Y}\). 
Let $\mathbf{r}_i=\hat{\mathbf{n}}_i - \mathbf{B}\hat{\mathbf{m}}_i$, we can rewrite~(\ref{measurement4}) to~(\ref{measurement5}):
\begin{equation}\label{measurement5}
 \hat{\mathbf{y}}_i=\mathbf{B} \hat{\mathbf{x}}_i + \hat{\mathbf{n}}_i - \mathbf{B}\hat{\mathbf{m}}_i=\mathbf{B} \hat{\mathbf{x}}_i-\mathbf{r}_i.
\end{equation}

This leads to the idea of of a simple least squares estimate \cite{IEEEhowto:Gleser}.
In general, the sensor measurement is noisy. However, in cases of high quality sensors, the noise power is very low. Hence, a simple regression approach has advantages, e.g., lower computational complexity and thus energy consumption.

As observed in \cite{IEEEhowto:Gleser}, if \(\boldsymbol{\Theta}\boldsymbol{\Theta}^T\) is non-singluar, the minimum over \(\mathbf{B}\) of \(\mathrm{tr}[(\mathbf{Y}-\mathbf{B}\boldsymbol{\Theta})(\mathbf{Y}-\mathbf{B}\boldsymbol{\Theta})^T]\) can be obtained based on the results of multivariate regression \(\mathbf{B}=\mathbf{Y}\boldsymbol{\Theta}^T[\boldsymbol{\Theta}\boldsymbol{\Theta}^T]^{-1}\). This solution corresponds to the search for \(\mathbf{B}\) with \(\mathbf{B}\boldsymbol{\Theta} = \mathbf{Y}\). 
The resulting matrix \(\mathbf{B}\) is obtained by multiplying \(\mathbf{Y}\) by the pseudoinverse of \(\boldsymbol{\Theta}\) (i. e., the Moore-Penrose inverse), i.e. \(\mathbf{B}=\mathbf{Y}\boldsymbol{\Theta}^T[\boldsymbol{\Theta}\boldsymbol{\Theta}^T]^{-1}\). This solution is the least squares solution. 
Interestingly, the stationary points of~(\ref{functionf}) are \(\boldsymbol{\Theta}=\mathbf{X}\) and \(\mathbf{Y}=\mathbf{B}\boldsymbol{\Theta}\).
\begin{proposition}\label{PropGrad}
If \(\boldsymbol{\Theta}\boldsymbol{\Theta}^T\) is a non-singular matrix, the stationary points of the function \(f(\mathbf{X},\mathbf{Y},\boldsymbol{\Theta})\) defined by~(\ref{functionf}) over \(\mathbf{B}\) and \(\boldsymbol{\Theta}\) is given by the solutions:
\begin{align*} \boldsymbol{\Theta}^{\star}&=\mathbf{X}\\
    \mathbf{B}^{\star}&=\mathbf{Y}\boldsymbol{\Theta}^T[\boldsymbol{\Theta}\boldsymbol{\Theta}^T]^{-1}.
\end{align*}
\end{proposition}
\begin{proof}
    See Appendix.
\end{proof}
According to Appendix, with \(\boldsymbol{\Theta}=\mathbf{X}\), we finally get a simple equation (\ref{gradB4}) for the MLE of \(\mathbf{B}\):
\begin{equation}\label{gradB4}
    \mathbf{B}=\mathbf{Y}\mathbf{X}^T[\mathbf{X}\mathbf{X}^T]^{-1}
\end{equation}
Finally, from \(\mathbf{B}\) with~(\ref{Augmented}), we obtain the \(\tilde{\mathbf{A}}\) estimatation of the matrix \(\mathbf{A}\) and the \(\tilde{\mathbf{b}}\) estimation of the vector \(\mathbf{b}\).
At the end, we are interested in a transformation of the measured data \(\mathbf{x}_i\) of system~\(1\) to system~\(2\) \(\tilde{\mathbf{y}}_i\) as described in~(\ref{transform5}):
\begin{equation}\label{transform5}
    \tilde{\mathbf{y}}_i=\tilde{\mathbf{A}}\mathbf{x}_i+\tilde{\mathbf{b}}.
\end{equation}
\begin{algorithm}
 \caption{Augmented implementation of Gleser and Watson~(denoised)}
 \textbf{Data:} Measurement data of the two systems \(\mathbf{X}\), \(\mathbf{Y}\)\\
  \textbf{Result:} Estimation of \(\tilde{\mathbf{A}}\), \(\tilde{\mathbf{b}}\) and \(\tilde{\boldsymbol{\Theta}}=\boldsymbol{\Theta}_{E}\)
 
\begin{algorithmic}[1]

\STATE \([\mathbf{U},\boldsymbol{\Lambda}]=\mathrm{eigs}(\mathbf{X}^T\mathbf{X}+\mathbf{Y}^T\mathbf{Y},p)\)\;

\STATE \(\boldsymbol{\Theta}^T=\mathbf{U}\mathbf{U}^T\mathbf{X}^T\)\;

\STATE \(\boldsymbol{\Theta}(p,:)=\mathbf{1}^T \)\;

\STATE \(\mathbf{B}=\mathbf{Y}\boldsymbol{\Theta}^T(\boldsymbol{\Theta}\boldsymbol{\Theta}^T)^{-1}\)\;

\STATE \(\tilde{\mathbf{A}}=\mathbf{B}(1:p-1,1:p-1)\)\;

\STATE \(\tilde{\mathbf{b}}=\mathbf{b}(1:p-1,p)\)\;

\STATE \(\boldsymbol{\Theta}_{E}=\boldsymbol{\Theta}(1:p-1,:)\)
\end{algorithmic}
\label{Alg1}
\end{algorithm}

Algorithm~\ref{Alg1} is the augmented implementation of Gleser and Watson. In line~\(2\) we calculate the~\(p\) largest eigenvalues \(\boldsymbol{\Lambda}\) according to Lemma~\ref{LemmaGleser}, and use the associated eigenvectors to estimate~\(\boldsymbol{\Theta}\), the origins of \(\mathbf{X}\), in line~\(3\). 
The variables $\gamma_i$ should satisfy  $\gamma_i=1$, $\forall i=1,...,n$. Therefore, after step~\(2\) of Algorithm~\ref{Alg1} set $\gamma_i=1$, $\forall i=1,...,n$.
The resulting~\(\boldsymbol{\Theta}\) augmentation line contains many values close to~\(1\), so we set those to the value \(1\) in line~\(4\) to suppress the small eigenvalues impact on~\(\boldsymbol{\Theta}\) with low noise power \cite{Bib:PrincipalComponents}. Lines~\(5\) to~\(7\) calculate and extract the estimated AT parameters and line~\(8\) removes the augmentation to format the origin estimation.

Algorithm~\ref{Alg2} corresponds to the derivation of the simple least squares according to Proposition~\ref{PropGrad}. Since our focus lies not in the estimation of the origins \(\boldsymbol{\Theta}\), we can directly calculate and extract the estimatation parameters in lines~\(2\) to~\(5\). Although we cannot achieve improved complexity, it should be possible to achieve lower energy consumption through an improved number of operations, since Alg.~\ref{Alg2} skips step~1 and step~2 of Alg.~\ref{Alg1}. Furthermore, the accuracy is very close to the results of Alg.~\ref{Alg3} (cf. Table~\ref{tab:simuerrors} and Table~\ref{tab:realerrors}), which, however, requires a similar number of operations as Alg.~\ref{Alg1}.
\begin{algorithm}[]
 \caption{Implementation of the simple least squares according to Proposition~\ref{PropGrad}.}
 \textbf{Data:} Measurement data of the two systems \(\mathbf{X}\), \(\mathbf{Y}\)\\
 \textbf{Result:} Estimation of \(\tilde{\mathbf{A}}\), \(\tilde{\mathbf{b}}\) and \(\tilde{\boldsymbol{\Theta}}=\boldsymbol{\Theta}_{E}\)

\begin{algorithmic}[1]

\STATE \(\mathbf{B}=\mathbf{Y}\mathbf{X}^T(\mathbf{X}\mathbf{X}^T)^{-1}\)\;

\STATE \(\tilde{\mathbf{A}}=\mathbf{B}(1:p-1,1:p-1)\)\;

\STATE \(\tilde{\mathbf{b}}=\mathbf{b}(1:p-1,p)\)\;

\STATE \(\boldsymbol{\Theta}_{E}=\mathbf{X}(1:p-1,:)\)
\end{algorithmic}

\label{Alg2}
\end{algorithm}

The simple least squares solution of Alg.~\ref{Alg2} and the solution of Gleser and Watson~in Alg.~\ref{Alg1} can be combined to the improved Alg.~\ref{Alg3} by using the following observations. 
\subsection{Hybrid Approach}\label{Regression}

\begin{proposition}\label{Prop2}
Let \(\mathbf{D}\) be the diagonal matrix of all eigenvalues and \(\mathbf{V}\) be the matrix of all corresponding eigenvectors of \(\mathbf{X}^T\mathbf{X}+\mathbf{Y}^T\mathbf{Y}\) and let \(\mathbf{B}^{\dagger}\) and \(\boldsymbol{\Theta}^{\dagger}\) denote the solution of \(\mathbf{B}\) in this case
then the estimates are given by the Proposition~\ref{PropGrad} and it holds \(\mathbf{B}^{\dagger}=\mathbf{B}^{\star}\) and \(\boldsymbol{\Theta}^{\dagger}=\boldsymbol{\Theta}^{\star}\).
\end{proposition}
\begin{proof}
The Matrix \(\mathbf{X}^T\mathbf{X}+\mathbf{Y}^T\mathbf{Y}\) is symmetric, then there exist an orthonormal basis of eigenvectors \cite{Golub}. According to the symmetric Schur Decomposition (see Theorem 8.1.1 in \cite{Golub}) there exists a real orthogonal matrix \(\mathbf{V}\) with \(\mathbf{V}\mathbf{V}^T=\mathbf{I}\) that satisfies~(\ref{eigenV})
\begin{equation}\label{eigenV}
   \mathbf{V}^T (\mathbf{X}^T\mathbf{X}+\mathbf{Y}^T\mathbf{Y})\mathbf{V}=\mathbf{D}.
\end{equation}
With \(\mathbf{V}\mathbf{V}^T = \mathbf{I}\), we can conclude \(\boldsymbol{\Theta}^T=\mathbf{V} \mathbf{V}^T\mathbf{X}^T=\mathbf{X}^T.\)
\end{proof}
According to Proposition~\ref{Prop2}, the results of Alg.~\ref{Alg1} and Alg.~\ref{Alg2} are identical if all eigenvalues are used. 
As Alg.~\ref{Alg1} offers a better estimation of \(\boldsymbol{\Theta}\), we can improve Alg.~\ref{Alg2}, by using the estimator for \(\boldsymbol{\Theta}\) from Alg.~\ref{Alg1}. The resulting (hybrid) algorithm is given in Alg.~\ref{Alg3}.
\begin{table*}[!t]
\centering
\caption{Averaged simulation errors \(\bar{e}_{x}\) and \(\bar{e}_{y}\) for implementation of Gleser and Watson, Alg.~\ref{Alg1}, Alg.~\ref{Alg2}, and Alg.~\ref{Alg3} with increasing standard deviation \(\sigma\).}
\begin{tabular}{ | c | r | r | r | r | r | r | r | r |}
  \hline
  \diagbox{Algorithm:}{Standard deviation \(\sigma\):} & 1 & 3 & 5 & 7 & 9 & 11 & 13 & 15 \\ \hline
Gleser and Watson~\(\bar{e}_{x}\) & \(\mathbf{1.0975}\) & \(\mathbf{3.2894}\) & \(\mathbf{5.4887}\) & \(\mathbf{7.6816}\) & \(\mathbf{9.8777}\) & \(\mathbf{12.0760}\) & \(\mathbf{14.2788}\) & \(\mathbf{16.4682}\) \\ \hline
Alg.~\ref{Alg1} \(\bar{e}_{x}\) & \(\mathbf{1.0975}\) & \(\mathbf{3.2894}\) & \(\mathbf{5.4887}\) & \(\mathbf{7.6816}\) & \(\mathbf{9.8777}\) & \(\mathbf{12.0760}\) & \(\mathbf{14.2788}\) & \(\mathbf{16.4682}\)  \\ \hline
Alg.~\ref{Alg2} \(\bar{e}_{x}\) & \(1.2539\) & \(3.7581\) & \(6.2647\) & \(8.7715\) & \(11.2802\) & \(13.7877\) & \(16.3000\) & \(18.8043\) \\ \hline
Alg.~\ref{Alg3} \(\bar{e}_{x}\) & \(\mathbf{1.0975}\) & \(\mathbf{3.2894}\) & \(\mathbf{5.4887}\) & \(\mathbf{7.6816}\) & \(\mathbf{9.8777}\) & \(\mathbf{12.0760}\) & \(\mathbf{14.2788}\) & \(\mathbf{16.4682}\)  \\ \hhline{|=|=|=|=|=|=|=|=|=|}
Gleser and Watson~\(\bar{e}_{y}\) & \(0.6479\) & \(2.0236\) & \(3.5964\) & \(5.4476\) & \(7.5847\) & \(10.0661\) & \(12.8498\) & \(15.9657\) \\ \hline
Alg.~\ref{Alg1} \(\bar{e}_{y}\) & \(\mathbf{0.6444}\) & \(1.9338\) & \(3.2165\) & \(4.4912\) & \(5.7641\) & \(7.0201\) & \(8.2726\) & \(9.4927\) \\ \hline
Alg.~\ref{Alg2} \(\bar{e}_{y}\) & \(0.8320\) & \(2.4870\) & \(4.1052\) & \(5.6797\) & \(7.1982\) & \(8.6297\) & \(9.9800\) & \(11.2375\) \\ \hline
Alg.~\ref{Alg3} \(\bar{e}_{y}\) & \(\mathbf{0.6444}\) & \(\mathbf{1.9322}\) & \(\mathbf{3.2110}\) & \(\mathbf{4.4739}\) & \(\mathbf{5.7232}\) & \(\mathbf{6.9369}\) & \(\mathbf{8.1163}\) & \(\mathbf{9.2513}\) \\ \hline
\end{tabular}
\label{tab:simuerrors}
\end{table*}
\begin{algorithm}[]
\caption{Hybrid solution: Implementation of a combination of the simple MLE and Gleser and Watson}\label{Alg3}
 \textbf{Data:} Measurement data of the two systems \(\mathbf{X}\), \(\mathbf{Y}\)\\
 \textbf{Result:} Estimation of \(\tilde{\mathbf{A}}\), \(\tilde{\mathbf{b}}\) and \(\tilde{\boldsymbol{\Theta}}=\boldsymbol{\Theta}_{E}\)

\begin{algorithmic}[1]

\STATE \([\mathbf{V},\mathbf{D}]=\mathrm{eigs}(\mathbf{X}^T\mathbf{X}+\mathbf{Y}^T\mathbf{Y})\)\;

\STATE \(\boldsymbol{\Theta}^T=\mathbf{V}\mathbf{V}^T\mathbf{X}^T\)\;
 
\STATE \(\mathbf{B}=\mathbf{Y}\mathbf{X}^T(\mathbf{X}\mathbf{X}^T)^{-1}\)\;

\STATE \(\tilde{\mathbf{A}}=\mathbf{B}(1:p-1,1:p-1)\)\;

\STATE \(\tilde{\mathbf{b}}=\mathbf{b}(1:p-1,p)\)\;

\STATE \(\boldsymbol{\Theta}_{E}=\boldsymbol{\Theta}(1:p-1,:)\)
\end{algorithmic}
\end{algorithm}
\section{Results}\label{Results}
\noindent We evaluate the presented algorithms and investigate the difference between the implementation of Gleser and Watson~and our denoised augmented improvement using Monte Carlo simulations. To test the applicability in practice, we provide the algorithms with data from an application for electronic noses.
\subsection{Monte Carlo Simulations}\label{MonteCarloSimulation}
\noindent We developed Monte Carlo simulations to evaluate the algorithms. In~\(l=1000\) simulation runs~\(n=1000\) data vectors~\(\boldsymbol{\theta}_i\) were generated with a data generation tool from \cite{Bib:DataPaper}. We add a noise vector, where each element is a normally distributed random number scaled by the square root of the desired variance \(\sigma^2\) using standard deviations \(\sigma\) from~\(1\) to~\(15\). The data vectors have a dimension of~\(q=2\). 
We used a fixed transformation given by:
\begin{equation*}
  \mathbf{A}=\left[\begin{array}{cc}
0.3430 & 0.3430 \\
0.1715 & 0.8575 \\
\end{array} \right]
\end{equation*}
and \(\mathbf{b}^T=[52,-58]\).
For comparison, the following errors were defined in~(\ref{ErrorMCS}):
\begin{equation}\label{ErrorMCS}
    e_{y}=\frac{1}{n}\sum_{i=1}^n \lVert (\tilde{\mathbf{A}}\tilde{\boldsymbol{\theta}}_i+\tilde{\mathbf{b}})-(\mathbf{A}\boldsymbol{\theta}_i+\mathbf{b}) \rVert_{2},
\end{equation}
where \(\tilde{\boldsymbol{\theta}}=[\tilde{\boldsymbol{\theta}}_1,...,\tilde{\boldsymbol{\theta}}_n]\).
\begin{figure}
     \centering
      \includegraphics[scale=0.3]{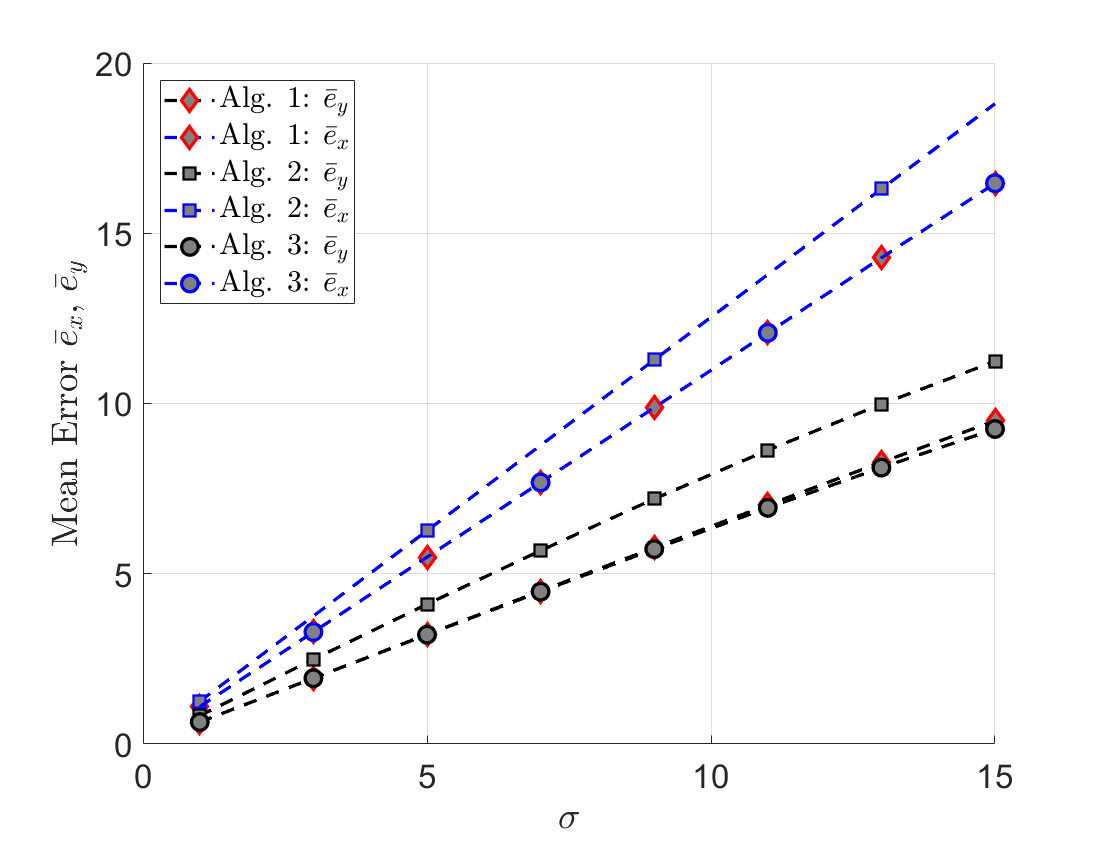}
         \caption{Simulation result, mean errors \(\bar{e}_{x}\) and \(\bar{e}_{y}\) for Alg.~\ref{Alg1}, Alg.~\ref{Alg2}, and Alg.~\ref{Alg3} as a function of the standard deviation \(\sigma\).}
         \label{fig:MeanError}
\end{figure}
The solution of Gleser and Watson~also estimates \(\boldsymbol{\Theta}\), therefore, we also calculate the error defined in~(\ref{errorX}):
\begin{equation}\label{errorX}
    e_{x}=\frac{1}{n}\sum_{i=1}^n \lVert \tilde{\boldsymbol{\theta}}_i-\boldsymbol{\theta}_i \rVert_{2}. 
\end{equation}
Finally, we calculate the arithmetic mean values of the errors~\(e_x\) and~\(e_y\) over the \(l=1000\) simulation runs. 

The arithmetic mean values are denoted with~\(\bar{e}_x\) and~\(\bar{e}_y\). Fig.~\ref{fig:MeanError} shows the mean estimation errors~\(\bar{e}_{x}\) and~\(\bar{e}_{y}\) of the algorithms. 
Since Algorithm~\ref{Alg2} only estimates the parameters of the transformation between the noisy measured data from system~1 to system~2, it performs worse than Gleser and Watson~taking the origins into account. In this case, the error~\(\bar{e}_x\) is only the noise of system~1. The hybrid version in Alg.~\ref{Alg3} benefits from Gleser and Watson~and delivers similar quality results for the given transformation matrix.

Analyzing error~\(\bar{e}_y\) in Table~\ref{tab:simuerrors} shows the improvement of the modification. Alg.~\ref{Alg1} improves the results of Gleser and Watson
\subsection{Eigenvector Denoising by Augmentation}\label{GleserAUSimlation}

\noindent Fig.~\ref{fig:gleserAU} shows the results of the Monte Carlo simulation for the eigenvector denoising extension of Alg.~\ref{Alg1} in comparison with the Gleser and Watson~standard implementation. One can see the improvement in the estimation of \(\bar{e}_{y}\) while error \(\bar{e}_{x}\) remains similar. 
\begin{table*}[!t]
\centering
\caption{Normalized transformation error \(\bar{e}_{y}\) linearly scaled (\(min=0.00072817\) resulting 0, \(max=0.0028\) resulting 1) using real data. Averaged comparison for each sensor dataset with the transformed data from all sensors for the implementation of Gleser and Watson, Alg.~\ref{Alg1}, Alg.~\ref{Alg2}, Alg.~\ref{Alg3}, and the feature-wise normalized dataset.}
\begin{tabular}{ | c | r | r | r | r | r | r | r | r |}
  \hline
  \diagbox{Algorithm:}{Source data sensor:} & 1 & 2 & 3 & 4 & 5 & 6 & 7 & 8 \\ \hline
Gleser and Watson~\(\bar{e}_{y}\) & \(0.0448\) & \(0.1170\) & \(0.0625\) & \(0.0075\) & \(0.2824\) & \(0.3171\) & \(0.0316\) & \(0.0186\) \\ \hline
Alg.~\ref{Alg1} \(\bar{e}_{y}\) & \(0.0448\) & \(0.1169\) & \(0.0625\) & \(0.0075\) & \(0.2822\) & \(0.3170\) & \(0.0316\) & \(0.0186\) \\ \hline
Alg.~\ref{Alg2} \(\bar{e}_{y}\) & \(0.0100\) & \(0.0510\) & \(0.0502\) & \(\mathbf{0}\) & \(0.1779\) & \(0.1208\) & \(0.0357\) & \(0.0117\) \\ \hline
Alg.~\ref{Alg3} \(\bar{e}_{y}\) & \(0.0100\) & \(0.0510\) & \(0.0502\) & \(\mathbf{0}\) & \(0.1779\) & \(0.1208\) & \(0.0357\) & \(0.0117\) \\ \hline
feature-wise normalized & \(0.9885\) & \(0.6450\) & \(0.5701\) & \(0.5016\) & \(0.4780\) & \(1\) & \(0.4636\) & \(0.4164\) \\ \hline
\end{tabular}
\label{tab:realerrors}
\end{table*}
Table~\ref{tab:simuerrors} gives an overview of the mean errors \(\bar{e}_{x}\) and \(\bar{e}_{y}\) for the implementation of Gleser and Watson, Alg.~\ref{Alg1}, Alg.~\ref{Alg2}, and Alg.~\ref{Alg3} for the selected values of the standard deviation \(\sigma\).
\subsection{Case Study: 8 Sensor Board Scenario}\label{RealDataSimulation}

\noindent For the collection of real test data we measured effects of humidity, temperature and disturbances in the air. In addition, we heat the samples during data recording and clean the sensors at regular intervals. 
\vspace{0.2em}

\noindent \textbf{Data Preparation:} The Bosch BME688 development kit board, shown in Fig.~\ref{fig:bme688}, is equipped with eight identical BME688 metal oxide sensors that are used for data recording. For the experiment, we prepared two samples~\cite{Bib:DatasetJunk}, water and apple juice, each with a volume of 100~ml. We present only the apple juice investigations for clarity. The results with water are similar and are available in the gitlab repository. Before each measurement of the sample, the sensor is cleaned in air. The measurement period per glass is~20 minutes. The temperature of the sample glass is increased by~1~°C every~2~minutes during the measurement. The total temperature increases from~25~°C at the beginning of the measurement, to~35~°C at the end of the measurement period. The Bosch BME688 is operated with a continuously repeating temperature cycle of~10 heating temperature steps, also called parallel mode. In this example, the temperature cycle is set to~5~steps at~200~°C and~5~steps at a sensor temperature of~400~°C, each at an interval of~1260~milliseconds. For later evaluation of the measurement data, the arithmetic mean is calculated from the~5~raw values at~200~°C and~400~°C. As already mentioned in section~\ref{cha:Introduction}, Fig.~\ref{fig:Intro1} shows the results of the estimated transformation for an normalized example data set in the use-case of an electronic nose measuring apple juice. The black dots show the position of the data points in system~\(1\), blue dots represent the measured data in system~2 and the red dots are the estimation of system~2 based on the calculated transformation of Proposition~\ref{PropGrad}. It can be seen that Proposition~\ref{PropGrad} allows a good estimation of the transformation in this case.
\vspace{0.2em}

\noindent \textbf{Data Evaluation:} We calculated the error~\({e}_{y}\) for the~\(64\)~combinations of sensors denoting~\(j\) as the source sensor and~\(k\) as the target sensor, \(j,k \in \{1,...,8\}\) by~(\ref{Error_with_X})
\begin{equation}
    \label{Error_with_X}
    {e}_{y}= \lVert (\tilde{\mathbf{A}}_{j,k}{\mathbf{x}_j}+\tilde{\mathbf{b}}_{j,k})-\mathbf{y}_k \rVert_{2},
\end{equation}
where the \(\tilde{\mathbf{A}}_{j,k}\) and \(\tilde{\mathbf{b}}_{j,k}\) are representing the transformation parameters from source sensor~\(j\) to target sensor~\(k\), including cases \(j=k\), using \(K=8\) sensors.
\begin{figure}
     \centering
      \includegraphics[scale=0.3]{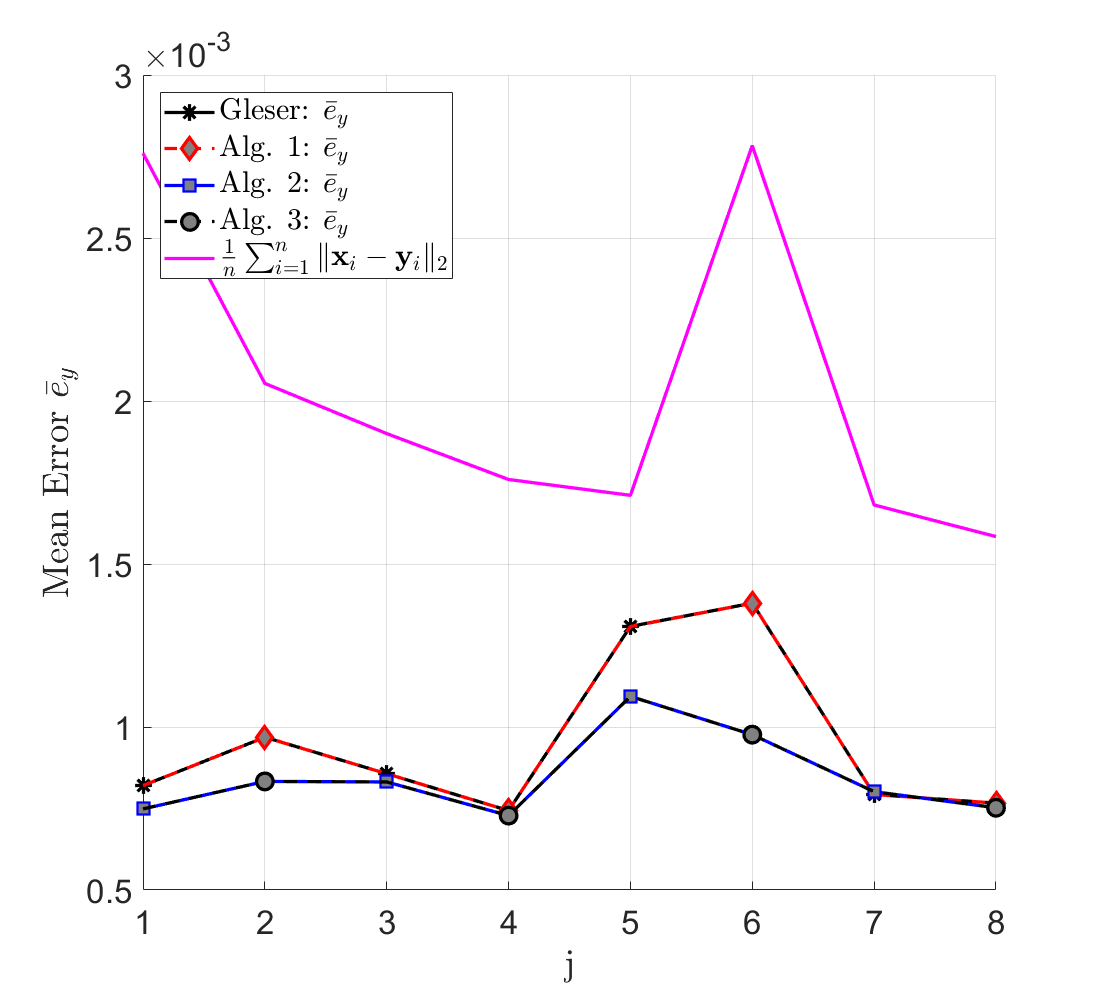}
         \caption{Experiment result, mean error \(\bar{e}_{y}\) for source sensor \(j\) to target sensor \(k\), \(j,k \in \{1,...,8\}\), compared to the distance between normalized source and target data. The peak at sensor~6 is likely due to deviations in the measurement procedure.}
         \label{fig:errorrealpercent}
\end{figure}

The results of the transformation error~\(\bar{e}_{y}\) for the normalized data set of the multi-sensor board are shown in fig.~\ref{fig:errorrealpercent}. The mean error for the~\(n=90\) data samples is calculated by~(\ref{errorRealym})
\begin{equation}\label{errorRealym}
    \bar{e}_{y}(j)= \frac{1}{K} \sum_{k=1}^K \frac{1}{n} \sum_{i=1}^n \lVert (\tilde{\mathbf{A}}_{j,k}{\mathbf{x}_{j}(i)}+\tilde{\mathbf{b}}_{j,k})-\mathbf{y}_{k}(i) \rVert_{2}.
\end{equation}

Since we can use the measured values \(\mathbf{y}_k\), we do not need to calculate them as in eq. \ref{ErrorMCS} by \((\mathbf{A}\boldsymbol{\theta}_i+\mathbf{b})\). Table~\ref{tab:realerrors} shows the normalized mean error \(\bar{e}_{y}\) for each source sensor, with the minimum value \(0.00072817\) resulting 0, and maximum value \(0.0028\) resulting 1. Since we focus on calculating the transformation from \(\mathbf{X}\) to \(\mathbf{Y}\), and do not try to estimate the origins of the data, the performance of Gleser and Watson~is slightly less accurate than the least squares, and the hybrid version delivers similar results as Proposition~\ref{PropGrad}. Our proposed Alg.~\ref{Alg2}, with reduced computational steps and therefore lower energy consumption, gives the best results except source sensor~6. However, this could also be due to inaccuracies in the measurement procedure and needs further investigation. A simple normalization of the data in our experiment leads to at least \(50~\%\) worse results than our proposed method.

\begin{figure}
     \centering
      \includegraphics[scale=0.27]{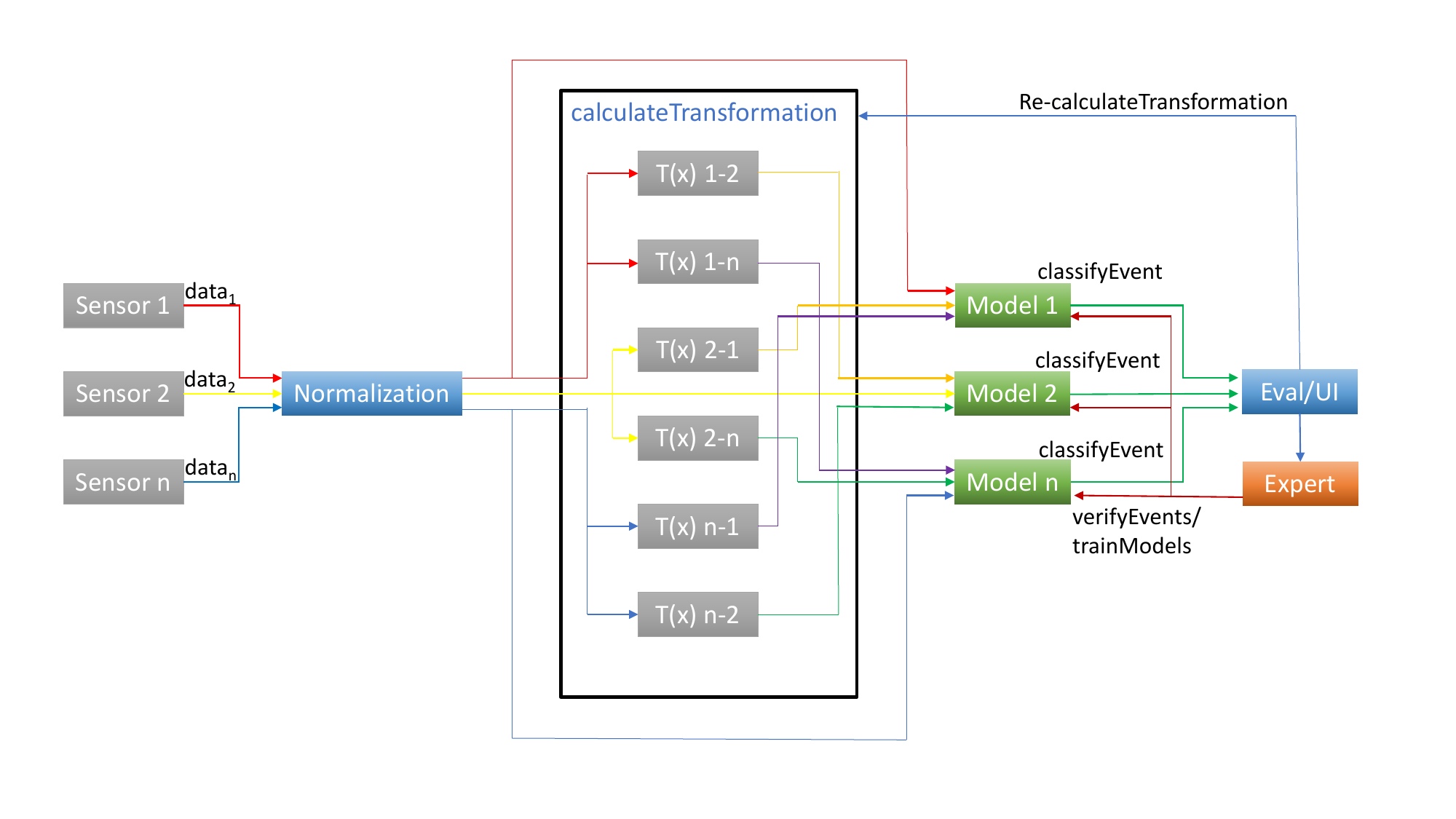}
         \caption{Concept of distributed learning using transformation. Each model learns from each (transformed) input data stream. Experts therefore review and train the estimation of transformation.}
         \label{fig:concept}
\end{figure}
\section{Discussion and Outlook}\label{Discussion}
\noindent In the article, we have shown that the procedures of estimating the 2-dimensional AT can be used for an expert-assisted learning between two systems. The simple direct estimation based on the stationary points of the parameters of the transformation can be combined with the solution for the estimation of the data from Gleser and Watson. This new hybrid solution provides the best estimates for both, simulations and real data. For the presented real data evaluation, we achieve best results with the simple Prop.~\ref{PropGrad}, but we are looking forward to investigate whether we can use the approach of Gleser and Watson~to improve the measurement values of our low-cost sensors in an extended setup. For the real data example, our simple least squares approach already achieves unique solution due to the non singularity in \(\boldsymbol{\Theta}\boldsymbol{\Theta}^T\), and the best accuracy with low computation, which is very promising for further developments in distributed learning. The results still need to be generalized for the m-dimensional case. 

Finally, in further work we will explore if only a part of the measured data is sufficient to give a good estimate of the transformation. The evaluation with multiple transformation can also be done in future work. Based on this, iterative methods can be developed in the future to estimate the transformation between different systems. Fig.~\ref{fig:concept} illustrates the concept of our future work to implement distributed learning based on the creation of similar models using the expert-supported transformation estimation (see the figure in the abstract) between the two data spaces. As to our knowledge, there is no similar approach, although it seems to be a simple way to match distributed systems and their events or anomalies. We plan to implement these concepts as the basis for a more efficient method of knowledge sharing between distributed nodes. We also plan to investigate in more detail why denoising improves the results. In our example, the fixed transformation we have chosen is favorable for Gleser and Watson However, when considering various alternatives for \(\mathbf{A}\), Gleser and Watson~produces notably inferior results.
We encourage readers to experiment with this option and assess its performance in different \(\mathbf{A}\) scenarios using a code provided on GitLab at~\url{https://gitlab.rlp.net/vski-engineering-group/paper/likelihood-based-sensor-calibration-for-expert-supported-distributed-learning-algorithms-in-iot-systems-estimating-affine-transformations/matlabcodeanddata}.

\section{Conclusion}
\noindent We propose an improved solution of Gleser and Watson~by augmentation of the matrix, denoising using the eigenvectors and the extension to an AT. We have shown that both solutions are connected by an eigenvalue decomposition. We further showed that a combination of both solutions provides better estimates for computing the parameters of the AT by both, simulation, and a real experiment. Finally we presented a novel concept for distributed learning using transformation with the support of expert knowledge.

All 3 algorithms offer a good approach to estimate the AT between similar systems. Further investigations should include special methods to fairly compare Gleser and Watson~with our hybrid proposition. The development of a simple but powerful expert-based distributed learning algorithm using ATs appears promising but requires further research in multidimensionality, and an iterative approach. Investigating more powerful transformation for stronger IoT nodes should make sense as well.


\section*{Appendix\\ \quad \quad Proof Prop. \ref{PropGrad}}
\renewcommand{\theequation}{A\arabic{equation}}
\setcounter{equation}{0}
\begin{proof}
    With the help of the gradient in~(\ref{eq119}) based on \cite{IEEEhowto:Petersen} (eq. (119))  
\begin{align}\label{eq119}
    \nabla_{\mathbf{X}}& \mathrm{tr}[(\mathbf{A}\mathbf{X}\mathbf{B}+\mathbf{C})(\mathbf{A}\mathbf{X}\mathbf{B}+\mathbf{C})^T]\\
    &=2\mathbf{A}^T(\mathbf{A}\mathbf{X}\mathbf{B}+\mathbf{C})\mathbf{B}^T,\nonumber
\end{align}
we can get the following gradients over \(\boldsymbol{\Theta}\)
\begin{align*}
   \nabla_{\boldsymbol{\Theta}} \mathrm{tr}[(\mathbf{X}-\boldsymbol{\Theta})(\mathbf{X}-\boldsymbol{\Theta})^T]
   =-2(\mathbf{X}-\boldsymbol{\Theta}) 
\end{align*}
and 
\begin{align*}
   \nabla_{\boldsymbol{\Theta}} \mathrm{tr}[(\mathbf{Y}-\mathbf{B}\boldsymbol{\Theta})(\mathbf{Y}-\mathbf{B}\boldsymbol{\Theta})^T]
   =-2\mathbf{B}^T(\mathbf{Y}-\mathbf{B}\boldsymbol{\Theta}). 
\end{align*}
With the gradients, we get the equation for the stationary points
\begin{equation*}
    \mathbf{0}=-2(\mathbf{X}-\boldsymbol{\Theta}) -2\mathbf{B}^T(\mathbf{Y}-\mathbf{B}\boldsymbol{\Theta}),
\end{equation*}
which results in~(\ref{gradTheta})
\begin{equation}\label{gradTheta}
\boldsymbol{\Theta}=\mathbf{X}+\mathbf{B}^T\mathbf{Y}-\mathbf{B}^T\mathbf{B}\boldsymbol{\Theta}.
\end{equation}
Similarly, we can derive the gradient over \(\mathbf{B}\)
\begin{align*}
   \nabla_{\mathbf{B}} \mathrm{tr}[(\mathbf{Y}-\mathbf{B}\boldsymbol{\Theta})(\mathbf{Y}-\mathbf{B}\boldsymbol{\Theta})^T]
   =-2(\mathbf{Y}-\mathbf{B}\boldsymbol{\Theta})\boldsymbol{\Theta}^T 
\end{align*}
which results in the~(\ref{gradB})
\begin{equation}\label{gradB}
    \mathbf{B}\boldsymbol{\Theta}\boldsymbol{\Theta}^T=\mathbf{Y}\boldsymbol{\Theta}^T.
\end{equation}
Multiplying both sides of~(\ref{gradB}) with \(\boldsymbol{\Theta}\) from right we get
\begin{equation*}
    \mathbf{B}\boldsymbol{\Theta}\boldsymbol{\Theta}^T\boldsymbol{\Theta}=\mathbf{Y}\boldsymbol{\Theta}^T\boldsymbol{\Theta}.
\end{equation*}
Multiplying both sides with \([\boldsymbol{\Theta}^T\boldsymbol{\Theta}]^{-1}\) from the right, we finally get~(\ref{gradB2})
\begin{equation}\label{gradB2}
    \mathbf{Y}=\mathbf{B}\boldsymbol{\Theta}. 
\end{equation}
Inserting~(\ref{gradB2}) in~(\ref{gradTheta}), we get~(\ref{gradB3})
\begin{equation}\label{gradB3}
    \boldsymbol{\Theta}=\mathbf{X}+\mathbf{B}^T\mathbf{B}\boldsymbol{\Theta}-\mathbf{B}^T\mathbf{B}\boldsymbol{\Theta}=\mathbf{X}.
\end{equation}
Multiplying both sides of~(\ref{gradB}) with \([\boldsymbol{\Theta}\boldsymbol{\Theta}^{T}]^{-1}\) from the right, we get result of Proposition~\ref{PropGrad}.
\end{proof}



\vfill

\end{document}